\documentclass{article}
\usepackage[utf8]{inputenc}

\usepackage[margin=1in]{geometry}
\usepackage{xcolor}
\usepackage{amsmath}
\usepackage{amssymb}
\usepackage{amsthm}
\usepackage{algorithm}
\usepackage{algpseudocode}
\usepackage{hyperref}
\usepackage{bbm}
\usepackage{graphicx}
\usepackage[font={small}]{caption}
\usepackage{thm-restate}
\usepackage[normalem]{ulem}

\newcommand{\nc}{\newcommand}
\nc{\DMO}{\DeclareMathOperator}
\newcommand{\N}{\mathbb{N}}
\newcommand{\cH}{\mathcal{H}}

\newcommand{\MX}{\mathcal{X}}
\newcommand{\One}{\mathbbm{1}}
\newcommand{\MH}{\cH}
\newcommand{\MG}{\mathcal{G}}
\newcommand{\ra}{\rightarrow}
\DeclareMathOperator{\Tdim}{Tdim}
\newcommand{\MF}{\mathcal{F}}
\newcommand{\MC}{\mathcal{J}}
\renewcommand{\MC}{\mathcal{C}}
\renewcommand{\^}[1]{^{(#1)}}
\newcommand{\MT}{\mathcal{T}}

\nc{\ep}{\epsilon}
\nc{\bx}{\mathbf{x}}
\nc{\bv}{\mathbf{v}}
\nc{\MN}{\mathcal{N}}
\DMO{\Ldim}{Ldim}
\nc{\BN}{\mathbb{N}}
\nc{\bw}{\mathbf{w}}

\nc{\noah}[1]{\textcolor{purple}{[Noah: #1]}}
\nc{\MV}{\mathcal{V}}
\nc{\MW}{\mathcal{W}}

\newtheorem{theorem}{Theorem}[section]
\newtheorem{observation}[theorem]{Observation}
\newtheorem{proposition}[theorem]{Proposition}
\newtheorem{lemma}[theorem]{Lemma}

\newtheorem{defn}[theorem]{Definition}

\allowdisplaybreaks

\title{Near-tight closure bounds for Littlestone and threshold dimensions}
\date{\today}
\author{Badih Ghazi\thanks{Google Research, Mountain View. \texttt{badihghazi@gmail.com, ravi.k53@gmail.com, pasin@google.com}.} \hspace*{0.5cm}
Noah Golowich\thanks{This work was done while interning at Google Research. MIT EECS. Supported at MIT by a Fannie \& John Hertz Foundation Fellowship, an MIT Akamai Fellowship, and an NSF Graduate Fellowship.   \texttt{nzg@mit.edu}.} \hspace*{0.5cm} 
Ravi Kumar\footnotemark[1]
\hspace*{0.5cm} 
Pasin Manurangsi\footnotemark[1]}

\begin{document}

\maketitle

\begin{abstract}
We study closure properties for the Littlestone and threshold dimensions of binary hypothesis classes. Given classes $\mathcal{H}_1, \ldots, \mathcal{H}_k$ of Boolean functions with bounded Littlestone (respectively, threshold) dimension, we establish an upper bound on the Littlestone (respectively, threshold) dimension of the class defined by applying an arbitrary binary aggregation rule to $\mathcal{H}_1, \ldots, \mathcal{H}_k$. We also show that our upper bounds are nearly tight. Our upper bounds give an exponential (in $k$) improvement upon analogous bounds shown by Alon et al. (COLT 2020), thus answering a question posed by their work.
\end{abstract}

\section{Introduction}
 Let $X$ be a set and $\MH_1, \ldots, \MH_k$ be hypothesis classes consisting of binary classifiers $h : X \ra \{0,1\}$; for instance, each of $\MH_1, \ldots, \MH_k$ may be a collection of experts. Given an arbitrary {\it aggregation rule} $G : \{0,1\}^k \ra \{0,1\}$ (e.g., the majority vote among the $k$ experts), we study the maximum possible complexity of the class $G(\MH_1, \ldots, \MH_k)$, defined as the set of all classifiers $x \mapsto G(h_1(x), \ldots, h_k(x))$ for some choices $h_1 \in \MH_1, \ldots, h_k \in \MH_k$, as a function of the complexities of $\MH_1, \ldots, \MH_k$. 

Such a closure property has long been known when complexity is measured via the VC dimension: Dudley \cite{dudley_central_1978} showed that if the VC dimension of each of $\MH_1, \ldots, \MH_k$ is at most $d$, then the VC dimension of $G(\MH_1, \ldots, \MH_k)$, is at most $O(dk \log k)$. Recently Alon et al.~\cite{alon_closure_2020} proved similar, but quantitatively weaker, closure properties for the {\it Littlestone dimension} \cite{littlestone_learning_1988} (Definition \ref{def:ldim}), which characterizes online learnability of a class \cite{ben-david_agnostic_2009}, and {\it threshold dimension} \cite{shelah_classification_1978,hodges_shorter_1997} (Definition \ref{def:tdim}), which is known to be exponentially related to Littlestone dimension and was used by Alon et al.~\cite{alon_private_2019} to show that privately PAC-learnable classes are online learnable (i.e., have finite Littlestone dimension). The upper bounds of \cite{alon_closure_2020} exhibit an exponential dependence on $k$, and it was asked in \cite{alon_closure_2020} whether this dependence could be improved. Our main contribution is to resolve this question in the affirmative, proving tighter upper bounds with a nearly linear dependence on $k$ 
and to show that this is nearly the best possible. In particular:
\begin{enumerate}
\item When the Littlestone dimension of each of  $\MH_1, \ldots, \MH_k$ is at most $d$, we show that the Littlestone dimension of $G(\MH_1, \ldots, \MH_k)$ is at most $O(dk \log k)$ (Proposition \ref{prop:ldim-ub}), improving upon the bound of $\tilde O(2^{2k} k^2 d)$ of \cite{alon_closure_2020}. Moreover, our upper bound is tight up to the $\log k$ factor (Observation \ref{obs:ldim-lb}).
\item When the threshold dimension of each of $\MH_1, \ldots, \MH_k$ is at most $d$, we show that the threshold dimension of $G(\MH_1, \ldots, \MH_k)$ is at most $2^{O(dk \log k)}$ (Proposition \ref{prop:tdim-ub}), and that it can be at least $2^{\Omega(dk)}$ (Proposition \ref{prop:tdim-lb}). These bounds improve upon the upper and lower bounds of $2^{2^{O(k)}d}$ and $2^{\Omega(d)}$, respectively, shown in  \cite{alon_closure_2020}. 
\end{enumerate}


\section{Closure bounds for Littlestone dimension}
\subsection{Preliminaries}
In this section we mostly follow the notation of \cite{rakhlin_statistical_2014,rakhlin_online_2015}.
For a positive integer $n$, we use $[n]$ to denote $\{1, \dots, n\}$.
For a positive integer $t$ and a sequence $\ep_1, \ep_2, \ldots, \ep_t, \ldots$, we let $\ep_{1:t}$ denote the tuple $(\ep_1, \ldots, \ep_t)$. As a convention let $\ep_{1:0}$ denote the empty sequence. Let $\{0,1\}^X$ be the set of all classifiers $f : X \ra \{0,1\}$.

For a set $X$, an {\it $X$-valued tree $\bx$} of depth $n$ is a collection of functions $\bx_t : \{-1,1\}^{t-1} \ra X$ for $1 \leq t \leq n$. Consider a binary hypothesis class $\MF \subseteq \{0,1\}^X$. 
The class $\MF$ is said to {\it shatter} a tree $\bx$ of depth $n$ if
$$
\forall (\ep_1, \ldots, \ep_n) \in \{-1,1\}^n, \ \exists f \in \MF \ \text{ s.t. } \ f(\bx_t(\ep_{1:t-1})) = \frac{\ep_t+1}{2} \ \forall t \in [n].
$$
\begin{defn}[Littlestone dimension]
\label{def:ldim}
The {\normalfont Littlestone dimension} of a class $\MF$, denoted $\Ldim(\MF)$, is the depth of the largest $X$-valued binary tree $\bx$ that is shattered by $\MF$. 
\end{defn}
A set $\MV$ of $\{0,1\}$-valued trees of depth $n$ is called a {\it 0-cover} for $\MF$ on a given $X$-valued tree $\bx$ of depth $n$ if:
$$
\forall f \in \MF, \ \forall (\ep_1, \ldots, \ep_{n-1}) \in \{-1,1\}^{n - 1}, \ \exists \bv \in \MV \ \ \text{ s.t. } f(\bx_t(\ep_{1:t-1})) = \bv_t(\ep_{1:t-1}) \ \forall t \in [n].
$$
The {\it 0-covering number} of $\MF$ on the tree $\bx$ is defined as:
$$
\MN_0(\MF, \bx) := \min \left\{ |\MV| : \text{$\MV$ is a 0-cover for $\MF$ on $\bx$} \right\}.
$$

\begin{lemma}[\cite{rakhlin_sequential_2014}, Theorem 7; ``Sauer--Shelah lemma for 0-covering number in trees'']
  \label{lem:sauer}
  For any $X$-valued tree $\bx$ of depth $n$, we have
  $$
\MN_0(\MF, \bx) \leq \sum_{i=0}^d {n \choose i} \leq \left( \frac{en}{d} \right)^d,
$$
when $\Ldim(\MF) = d < \infty$.
\end{lemma}


\subsection{Improved bounds}

Let $X$ be a set. 
For a function $G : \{0,1\}^k \ra \{0,1\}$ and classifiers $h_1, \ldots, h_k : X \ra \{0,1\}$, let $G(h_1, \ldots, h_k) : X \ra \{0,1\}$ be the mapping $x \mapsto G(h_1(x), \ldots, h_k(x))$. Then for binary hypothesis classes $\MH_1, \ldots, \MH_k \subseteq \{0,1\}^X$, we define
$$
G(\MH_1, \ldots, \MH_k) := \{ G(h_1, \ldots, h_k) : h_i \in \MH_i \}.
$$

In Proposition \ref{prop:ldim-ub}, we prove an upper bound for $\Ldim(G(\MH_1, \ldots, \MH_k))$ in terms of $\max_{1 \leq j \leq k} \Ldim(\MH_j)$ that grows nearly linearly with $k$. 
The proof follows as a consequence of the bound on the 0-covering number given by the  Sauer--Shelah lemma for trees (Lemma \ref{lem:sauer}), in a similar manner to Dudley's \cite[Proposition 7.12]{dudley_central_1978} proof of a closure property for VC classes using the classic Sauer--Shelah lemma \cite{sauer_density_1972,vapnik_uniform_1968}. 
In Section 2.1.1 of \cite{alon_closure_2020}, the authors state that they are not aware of a proof of Proposition \ref{prop:ldim-ub}  using the related definition of {\it thicket shatter function}. We discuss the relation between 0-covering number and thicket shatter function further in Section \ref{sec:thicket-shattering}. 
\begin{proposition}
\label{prop:ldim-ub}
  Let $G : \{0,1\}^k \ra \{0,1\}$ be a Boolean function, let $\MH_1, \ldots, \MH_k \subseteq \{0,1\}^X$ be binary hypothesis classes, and let $d \in \BN$ be such that $\Ldim(\MH_i) \leq d$ for all $i \in [k]$. Then
  $$
\Ldim(G(\MH_1, \ldots, \MH_k)) \leq  O(kd\log k).
  $$
\end{proposition}

Before proving Proposition \ref{prop:ldim-ub} we state the following lemma, which will be established as a corollary of Lemma \ref{lem:lower-bound-shattering} in Section \ref{sec:thicket-shattering}.
\begin{lemma}
\label{lem:lower-bound-shattering-cor}
  Suppose that $\MF \subset \{0,1\}^X$ shatters a tree $\bx$ of depth $n$. Then any 0-cover $\MV$ for $\MF$ on the tree $\bx$ has size at least $2^n$.
\end{lemma}
\begin{proof}[Proof of Proposition \ref{prop:ldim-ub}]
  It is without loss of generality to assume $d \geq 3$. 
  Let us write $N = \Ldim(G(\MH_1, \ldots, \MH_k))$. Let $\bx$ be an $X$-valued complete binary tree of depth $N$ that is shattered by $\MF$. By Lemma \ref{lem:sauer}, for each $i \in [k]$, since $\Ldim(\MH_i) \leq d$ for each $i$, we have
  $$
\MN_0(\MH_i, \bx) \leq \sum_{i=0}^d {N \choose i} \leq \left( \frac{eN}{d} \right)^d. 
  $$
Now, for each $i \in [k]$, let $\MV_i$ be a 0-cover for $\MH_i$ on $\bx$ of size $|\MV_i| \leq (eN/d)^d$.

We next construct a 0-cover for $G(\MH_1, \ldots, \MH_k)$ of size at most $\prod_{i=1}^k |\MV_i|$ as follows: for each tuple $\tau = (\bv\^1, \ldots, \bv\^k) \in \MV_1 \times \cdots \times \MV_k$, construct a tree $\bw\^\tau$ defined by
\begin{align*}
\bw\^{\tau}_t(\ep_{1:t-1}) := G(\bv\^1_t(\ep_{1:t-1}), \ldots, \bv\^k_t(\ep_{1:t-1})) & & \forall (\ep_1, \dots, \ep_{N-1}) \in \{-1, 1\}^N, t \in [N].
\end{align*}
To see that the collection $\MW := \{\bw\^\tau\}_{\tau \in \MV_1 \times \cdots \times \MV_k}$ indeed forms a 0-cover, consider any $g \in G(\MH_1, \ldots, \MH_k)$. Then there are $h_1 \in \MH_1, \ldots, h_k \in \MH_k$ so that $g(x) = G(h_1(x), \ldots, h_k(x))$. Also fix any sequence $(\ep_1, \ldots, \ep_{N-1}) \in \{-1,1\}^{N - 1}$. Since $\MV_i$ is a 0-cover for $\MH_i$ on $\bx$, for each $i \in [k]$, there is some $\bv\^i \in \MV_i$ so that for all $t \in [N]$, $h_i(\bx_t(\ep_{1:t-1})) = \bv\^i_t(\ep_{1:t-1})$. Thus, for $\tau = (\bv\^1, \ldots, \bv\^k)$, for each $t \in [N]$, we have
$$
\bw_t\^\tau(\ep_{1:t-1}) = G(\bv\^1_t(\ep_{1:t-1}), \ldots, \bv\^k_t(\ep_{1:t-1})) = G(h_1(\bx_t(\ep_{1:t-1})), \ldots, h_k(\bx_t(\ep_{1:t-1}))) = g(\bx_t(\ep_{1:t-1})).
$$
Hence $\MW$ is indeed a 0-cover of $G(\MH_1, \ldots, \MH_k)$ on $\bx$.

Since $\bx$ is shattered by $G(\MH_1, \ldots, \MH_k)$, by Lemma \ref{lem:lower-bound-shattering-cor}, we have that $|\MW| \geq 2^N$. Summarizing,
$$
2^N \leq |\MW| \leq (eN/d)^{kd},
$$
so $N \leq kd \log(eN/d)$, i.e., $N \leq  O(kd\log k)$. 
\end{proof}
We remark that an alternative way to upper bound $\Ldim(G(\MH_1, \ldots, \MH_k))$ in the context of Proposition \ref{prop:ldim-ub} is to use Proposition 9 and Corollary 6 of \cite{rakhlin_online_2015}. In particular, Corollary 6 of \cite{rakhlin_online_2015} gives a similar closure property for the sequential Rademacher complexities, and Proposition 9 of \cite{rakhlin_online_2015} implies that sequential Rademacher complexities are closely related to Littlestone dimension. However, this technique would give an upper bound of $O(\log^4(kd) \cdot k^2d)$, which is worse than that of Proposition \ref{prop:ldim-ub}.

We next point out that known Littlestone dimension lower bounds for the class of $k$-literal monotone disjunctions imply that our bound is the best possible up to a logarithmic (in $k$) factor (for the regime of parameter $k \leq 2^{0.9d}$):
\begin{observation} \label{obs:ldim-lb}
For any positive integers $k,d$, there is a domain $X$ and a class $\MH \subset \{0,1\}^X$ so that:
\begin{enumerate}
\item $\Ldim(\MH) \leq d$.
\item Defining $G : \{0,1\}^k \ra \{0,1\}$ to be the $k$-wise OR function, $\Ldim(G(\MH, \ldots, \MH)) \geq k \lfloor d - \log k\rfloor$. 
\end{enumerate}
\end{observation}

\begin{proof}
Let $D = 2^d$, $X = \{0,1\}^D$, and $\MH := \{ x \mapsto x_j : j \in [D]\}$, where $x \in \{0,1\}^D$. We have $\Ldim(\MH) \leq \lceil \log |\MH| \rceil = \lceil \log D \rceil = d$ as desired. Next, observe that the class $\cH' := G(\MH, \dots, \MH)$ is exactly the class of $k$-literal monotone disjunction, i.e., $\{ x \mapsto x_{i_1} \vee \cdots \vee x_{i_k} : i _1, \dots, i_k \in [D] \}$. Littlestone~\cite[Theorem 8]{littlestone_learning_1988} showed that $\Ldim(\cH') \geq k \lfloor \log(D/k) \rfloor = k \lfloor d - \log k \rfloor$, which completes the proof.
\end{proof}

\subsection{Tree covering numbers vs.~thicket shatter function}
\label{sec:thicket-shattering}
In this section we discuss an alternative to the 0-covering number for which a Sauer--Shelah lemma holds as well; we also establish that this  is strictly weaker than Lemma \ref{lem:sauer}, the Sauer--Shelah lemma for the 0-covering number for trees. This alternative to the  0-covering number is known as the thicket shatter function \cite{bhaskar_thicket_2017}:
\begin{defn}[Thicket shatter function]
  For an $X$-valued tree $\bx$ and function class $\MF$, let $\rho(\MF, \bx)$ denote the number of sequences $\ep = (\ep_1, \ldots, \ep_n) \in \{-1,1\}^n$ so that there is some $f \in \MF$ with
  \begin{equation}
    \label{eq:admissible-eps}
f(\bx_t(\ep_{1:t-1})) = \frac{\ep_t + 1}{2} \ \ \forall t \in \{1, 2, \ldots, n\}.
  \end{equation}
  In the event that (\ref{eq:admissible-eps}) holds, we will say that the sequence $\ep$ {\normalfont admits a solution in $\MF$ for the tree $\bx$}.
\end{defn}
Analogously to Lemma \ref{lem:sauer}, Bhaskar \cite[Theorem 4.1]{bhaskar_thicket_2017} showed that if the Littlestone dimension of $\MF$ is at most $d$, then for any tree $\bx$ of depth $n$, we have $\rho(\MF, \bx) \leq \sum_{i=0}^d {n \choose i}$. Lemma \ref{lem:lower-bound-shattering} shows that this statement is weaker than (i.e., follows directly from) Lemma \ref{lem:sauer}.
\begin{lemma}
  \label{lem:lower-bound-shattering}
For an $X$-valued tree $\bx$ and a function class $\MF \subset \{0,1\}^X$, it holds that $\rho(\MF, \bx) \leq \MN_0(\MF, \bx)$. 
\end{lemma}
\begin{proof}
  Let us give $\{-1,1\}^n$ the lexicographic ordering with $(-1, \ldots, -1)$ first, $(-1, \ldots, -1, 1)$ second, $(-1, \ldots, 1, -1)$ third, and so on. Let $\MV$ be a 0-cover for $\MF$ on the tree $\bx$.

  Fix any sequence $\ep = (\ep_1, \ldots, \ep_n) \in \{-1,1\}^n$ that admits a solution in $\MF$ (i.e., (\ref{eq:admissible-eps}) holds). There must be some $\bv\^\ep \in \MV$ so that for $t \in [n]$, we have $\bv\^\ep_t(\ep_{1:t-1}) = \frac{\ep_t + 1}{2}$. Fix any $\ep' < \ep$ (using the lexicographic ordering) which also admits a solution in $\MF$, and choose $t_0$ as small as possible so that $\ep'_{t_0} < \ep_{t_0}$. For all $t < t_0$, it follows that $\ep'_t = \ep_t$. Then we have
  $$
\bv_{t_0}\^{\ep'} (\ep_{1:t_0-1}) = 0 \neq 1 = \bv_{t_0}\^{\ep}(\ep_{1:t_0-1}).
$$
Hence $\bv\^\ep \neq \bv\^{\ep'}$, and hence, for all $\ep \in \{-1,1\}^n$ admitting a solution in $\MF$, the $\bv\^\ep$ are distinct.
\end{proof}
As an immediate corollary of Lemma \ref{lem:lower-bound-shattering}, we obtain  Lemma \ref{lem:lower-bound-shattering-cor}, since a tree $\bx$ that is shattered by $\MF$ satisfies $\rho(\MF, \bx) = 2^n$.

Finally, we show in Proposition \ref{prop:thicket-cover} that $\rho(\MF, \bx)$ and $\MN_0(\MF, \bx)$ may be very far apart. Though this fact is not used to prove any other results in this note, it establishes that Lemma \ref{lem:sauer} is in fact {\it strictly} stronger than \cite[Theorem 4.1]{bhaskar_thicket_2017}. This additional strength seems to be crucial in allowing us to establish Proposition \ref{prop:ldim-ub} using Lemma \ref{lem:sauer} (but not using $\rho(\MF, \bx) \leq \sum_{i=0}^d {n \choose i}$ alone). 
\begin{proposition}
  \label{prop:thicket-cover}
For any $n \in \BN$, there is a function class $\MF$ and a tree $\bx$ of depth $n$ so that $\rho(\MF, \bx) = 1$ yet $\MN_0(\MF, \bx) \geq 2^{n-1}$.
\end{proposition}
\begin{proof}
  Let us label all $2^{n}-1$ nodes of the tree $\bx$ with different elements of $X$; in particular, for each $1 \leq t \leq n$, denote the $2^{t-1}$ nodes of layer $t$ by $x_{t,1}, \ldots, x_{t,2^{t-1}}$, with all $x_{t,j}$ distinct. For simplicity we may assume that $X = \{ x_{t,j} : t \in [n], 1 \le j \leq 2^{t-1}\}$.  Now, choose $\MF$ to be the set of all functions $f : X \ra \{0,1\}$ so that $f(x_{1,1}) = f(x_{2,1}) = \cdots= f(x_{n,1}) = 0$. Then $\rho(\MF, \bx) = 1$ since the only $\ep$ admitting a solution in $\MF$ for the tree $\bx$ (i.e., satisfying (\ref{eq:admissible-eps})) is $\ep = (-1, \ldots, -1)$.

  On the other hand, letting $\ep_1 := 1$, then for any $\ep_2, \ldots, \ep_n \in \{-1,1\}$, there is some $f \in \MF$ so that
  $$
f(x_{1,1}) = f(\bx_1) = 0, \quad f(\bx_t(\ep_{1:t-1})) = \frac{1 + \ep_t}{2} \quad \forall t \geq 2.
$$
Now the argument of Lemma \ref{lem:lower-bound-shattering} establishes that there must be a unique element of a 0-cover for each sequence of the form $(1, \ep_2, \ldots, \ep_n)$. Thus $\MN_0(\MF, \bx) \geq 2^{n-1}$. 
\end{proof}

\section{Closure bounds for threshold dimension}
For positive integers $i,j$, write $\One[i \geq j]$ to be 1 if $i \geq j$ and 0 otherwise. Similarly write $\One[i = j]$ to be 1 if $i = j$ and 0 otherwise. The threshold dimension of a hypothesis class is defined as follows:
\begin{defn}[Threshold dimension]
\label{def:tdim}
For a binary hypothesis class $\MF \subset \{0,1\}^X$, the {\normalfont threshold dimension} of $\MF$, denoted $\Tdim(\MF)$, is the largest positive integer $d$ so that there are $x_1, \ldots, x_d \in X$ and $f_1, \ldots, f_d \in \MF$ such that $f_i(x_j) = \One[i \geq j]$ for all $i,j \in [d]$. In such a case, we say that $x_1, \ldots, x_d$ are {\normalfont threshold shattered by $\MF$ via} $f_1, \ldots, f_d$. 
\end{defn}

Proposition \ref{prop:tdim-ub} establishes an upper bound for $\Tdim(G(\MH_1, \ldots, \MH_k))$ in terms of $\max_{1 \leq j \leq k} \Tdim(\MH_j)$. 
It improves upon an upper bound of \cite{alon_closure_2020} that grows {\it doubly} exponentially in $k$. The proof technique is similar to that of \cite{alon_closure_2020}, except that in the application of Ramsey's theorem a coloring with only $2k$, as opposed to $2^{2k}$, colors is used.
\begin{proposition}
\label{prop:tdim-ub}
Let $G: \{0, 1\}^k \to \{0, 1\}$ be a Boolean function. Let $\cH_1, \dots, \cH_k \subseteq \{0, 1\}^X$ be binary hypothesis classes, and let $d \in \N$ be such that $\Tdim(\cH_i) \leq d$ for all $i \in [k]$. Then
\begin{align*}
\Tdim(G(\cH_1, \dots, \cH_k)) \leq 2^{O(k d \log k)}.
\end{align*}
\end{proposition}

\begin{proof}
Let $N$ be the smallest positive integer such that, for every coloring of the edges of the complete graph $K_N$ in $c = 2k$ colors, there exists a monochromatic clique of size $r = 2d + 1$. It is well known in Ramsey theory (e.g.,~\cite{greenwood_gleason_1955}) that $N \leq c^{r c} = 2^{O(kd \log k)}$. We will show that $\Tdim(G(\cH_1, \dots, \cH_k)) < N$.

Suppose contrapositively that $\Tdim(G(\cH_1, \dots, \cH_k)) \geq N$. By definition of threshold dimension, there exists $x_1, \dots, x_N \in X$ and $h_{i \ell} \in \cH_\ell$ for $i \in [N], \ell \in [k]$ such that
\begin{align*}
G(h_{i1}(x_j), \dots, h_{ik}(x_j)) = \One[i \geq j] & & \forall i, j \in [N].
\end{align*}
Consider the complete graph $K_N$ and a coloring with $2k$ colors defined as follows: for each $1 \leq p < q \leq N$, let $\ell \in [k]$ be the smallest index such that $h_{p\ell}(x_q) \ne h_{q\ell}(x_p)$; such $\ell$ must exist because $G(h_{p1}(x_q), \dots, h_{pk}(x_q)) = 0 \neq 1 = G(h_{q1}(x_p), \dots, h_{qk}(x_p))$. Then, let the color of the edge $\{p, q\}$ be $(\ell, h_{p\ell}(x_q))$.

By our choice of $N$, the graph must contain a monochromatic clique with vertices $i_1 < \dots < i_{2d + 1}$; let the color of its edges be $(t, y)$ where $t \in [k]$ and $y \in \{0, 1\}$. From how each edge's color is defined, the following holds for all distinct $u, v \in [2d + 1]$:
\begin{align*}
h_{i_u t}(x_{i_v}) =
\begin{cases}
y &\text{ if } u < v \\
1 - y &\text{ if } u > v.
\end{cases}
\end{align*}
Thus, if $y = 0$, then $x_{i_2}, x_{i_4}, \dots, x_{i_{2d}}$ is threshold shattered by $\cH_t$ (via the hypotheses $h_{i_3 t}, h_{i_5 t}, \dots, h_{i_{2d + 1} t}$). Otherwise, if $y = 1$, then $x_{i_{2d}}, x_{i_{2d - 2}}, \dots, x_{i_2}$ is threshold shattered by $\cH_t$ (via $h_{i_{2d - 1} t}, h_{i_{2d - 1} t}, \dots, h_{i_1 t}$). 
In both cases, we have $\Tdim(\cH_t) \geq d$, which concludes our proof.
\end{proof}

Next we establish a lower bound showing that Proposition \ref{prop:tdim-ub} is nearly tight. We need the following lemma from~\cite{alon_closure_2020}, which shows exponential dependence in $d$ (but not necessarily in $k$) is necessary.
\begin{lemma}[\cite{alon_closure_2020}, Theorem 2.2]
\label{lem:lb-exp-d}
For every $d \geq 6$ there is a class $\MC$ consisting of classifiers $f : \{0, 1, \ldots, 2^{\lfloor d/5 \rfloor}-1\} \ra \{0,1\}$ so that $\Tdim(\MC) \leq d$ yet 
$$
\Tdim(\{ f_1 \vee f_2 : f_1, f_2 \in \MC \}) = 2^{\lfloor d/5 \rfloor}.
$$
In fact, the class $\{ f_1 \vee f_2 : f_1, f_2 \in \MC \}$ realizes the thresholds $x \mapsto \One[b \geq x]$, for each $0 \leq b \leq 2^{\lfloor d/5\rfloor}-1$. 
\end{lemma}

Proposition \ref{prop:tdim-lb} shows that Proposition \ref{prop:tdim-ub} is tight, up to possibly the factor of $\log k$ in the exponent:
\begin{proposition}
\label{prop:tdim-lb}
For any positive integers $d \geq 6$ and $k$, there is a domain $\MX$ and classes $\MH_1, \ldots, \MH_{3k} : \MX \ra \{0,1\}$ and a function $G : \{0,1\}^{3k} \ra \{0,1\}$ so that:
\begin{enumerate}
\item $\max\{\Tdim(\MH_1), \ldots ,\Tdim(\MH_{3k})\} \leq d$.
\item $\Tdim(G(\MH_1, \ldots, \MH_{3k})) = 2^{k \lfloor d/5 \rfloor}$.
\end{enumerate}
\end{proposition}
\begin{proof}
Fix $d \geq 6, k$ and write $D := 2^{\lfloor d/5 \rfloor}$. Consider the domain $\MX := \{0, 1, \ldots, D^k-1\}$. For $x \in \MX$, we will write its base-$D$ representation as $x = x_1 x_2 \cdots x_k$, so that $x_1, \ldots, x_k \in \{0, 1, \ldots D-1 \}$. Let $\MC$, consisting of functions $f : \{0,1, \ldots, D-1\} \ra \{0,1\}$, be the class from Lemma \ref{lem:lb-exp-d}. We now define $k$ classes $\MH_1, \ldots, \MH_k$, as follows: for $1 \leq j \leq k$, let $\MH_j := \{ h_{j,f} : f \in \MC \}$, where for $f \in \MC$, 
\begin{align}
h_{j,f}(x_1 \cdots x_k) &= f(x_j).
\end{align}
Also define classes $\MG_1, \ldots, \MG_k$ as follows: for $j \in [k]$, let $\MG_j := \{ g_{j,0}, \ldots, g_{j,D-1}\}$, where for $b \in \{0, 1, \dots, D - 1\}$,
$$
g_{j, b}(x_1 \cdots x_k) = \One[x_j = b].
$$
Now define $\tilde G : \{0,1\}^{2k} \ra \{0,1\}$ as follows:
$ \tilde G(y_1, \ldots, y_k, z_1, \ldots, z_k) = 1$ if and only if either (a) $y_1 = \cdots = y_k = 1$ or (b) in the case that there is a smallest index $j$ with $z_j = 0$, it holds that $y_1 = \cdots = y_j = 1$. Finally define $G : \{0,1\}^{3k} \ra \{0,1\}$, as follows:
$$
G(y_1, \ldots, y_k, y_1', \ldots, y_k', z_1, \ldots, z_k) = \tilde G(y_1 \vee y_1', \ldots, y_k \vee y_k', z_1, \ldots, z_k).
$$

On one hand, it is straightforward to see that for each $j \in [k]$, $\Tdim(\MH_j) \leq d$, since $\Tdim(\MC) \leq d$ from Lemma \ref{lem:lb-exp-d}. It is also straightforward that $\Tdim(\MG_j) \leq 1$ for each $j$: if the threshold dimension were at least 2, then there would be $x\^1, x\^2 \in \MX$ that are threshold shattered via $g_{j, b_1}, g_{j, b_2}$ for some $b_1, b_2 \in \{0,1, \ldots, D-1\}$.
However, $g_{j, b_1}(x\^1) = g_{j, b_2}(x\^1) = 1$ implies that $b_1 = x\^1_j = b_2$ which contradicts $g_{j, b_1}(x\^2) = 0 \ne 1 = g_{j, b_2}(x\^2)$.

On the other hand, we claim that $\Tdim(G(\MH_1, \ldots, \MH_k, \MH_1, \ldots, \MH_k ,\MG_1, \ldots, \MG_k)) \geq D^k$. Now let $\MT_j := \{\tau_{j,0}, \ldots, \tau_{j,D - 1}\}$, where $\tau_{j,b}(x) = \One[b \geq x_j]$. Notice that
$$
G(\MH_1, \ldots, \MH_k, \MH_1, \ldots, \MH_k, \MG_1, \ldots, \MG_k) = \tilde G(\MH_1 \vee \MH_1, \ldots, \MH_k \vee \MH_k, \MG_1, \ldots, \MG_k) \supseteq \tilde G(\MT_1, \ldots, \MT_j, \MG_1, \ldots, \MG_k),
$$
where the inclusion above follows from Lemma \ref{lem:lb-exp-d}: indeed, the lemma implies that for each $b \in \{0,1,\ldots, D-1\}$, there are some $f_1, f_2 \in \MC$ so that $(f_1 \vee f_2)(\cdot) = \One[b \geq \cdot]$. In particular, it follows that $(h_{j,f_1} \vee h_{j,f_2})(x) = \One[b \geq x_j] = \tau_{j,b}(x)$.

It therefore suffices to show that $\tilde G(\MT_1, \ldots, \MT_k, \MG_1, \ldots, \MG_k)$ can realize all threshold functions on $\MX$. Indeed, for any $a = a_1 \cdots a_k \in \MX$, for $a_1, \ldots, a_k \in \{0,1,\ldots, D-1\}$, we have, for each $x \in \MX$,
$$
\One[a \geq x] = \tilde G (\tau_{1,a_1}(x), \ldots, \tau_{k,a_k}(x), g_{1,a_1}(x), \ldots, g_{k,a_k}(x)).
$$
To see that the above holds, simply note that $a \geq x$ if and only if either (a) $a_j \geq x_j$ for each $1 \leq j \leq k$, or (b) for the smallest $j$ such that $a_j \neq x_j$, we have $a_{j'} \geq x_{j'}$ for $1 \leq j' \leq j$.
\end{proof}

\section{Future work}
There is a gap of $\log k$ between the upper bound of Proposition \ref{prop:ldim-ub} and the lower bound of Observation \ref{obs:ldim-lb}, as well as between the exponents in the upper bound of Proposition \ref{prop:tdim-ub} and the lower bound of Proposition \ref{prop:tdim-lb}. It would be interesting to close these gaps. To the best of our knowledge a similar gap exists for the VC dimension \cite{dudley_central_1978}.

\cite{alon_closure_2020} additionally established a similar closure property to the ones considered in this note for the sample complexity of private PAC learning. Their upper bound has a polynomial dependence on $k$; it would be interesting to determine if a stronger upper bound (say, nearly linear in $k$) could be established.

\bibliographystyle{alpha}
\bibliography{refs}

\end{document}